\pdfoutput=1

\documentclass{article}

\PassOptionsToPackage{numbers, compress}{natbib}
\usepackage[preprint]{main}

\usepackage[utf8]{inputenc} %
\usepackage[T1]{fontenc}    %
\usepackage{hyperref}       %
\usepackage{url}            %
\usepackage{booktabs}       %
\usepackage{amsfonts}       %
\usepackage{nicefrac}       %
\usepackage{microtype}      %
\usepackage{xcolor}         %

\usepackage{graphicx}
\usepackage{amssymb}
\usepackage{amsmath}
\usepackage{subcaption}
\usepackage{float}
\usepackage{amsthm}
\usepackage[capitalize,noabbrev]{cleveref}

\newtheorem{theorem}{Theorem}
\newtheorem{lemma}[theorem]{Lemma}
\newtheorem{proposition}[theorem]{Proposition}

\newcommand{\ie}{\textit{i.e.,~}}
\newcommand{\eg}{\textit{e.g.,~}}

\title{Gaussian Weight Sampling for Scalable, Efficient and Stable Pseudo-Quantization Training}

\author{Myeonghwan Ahn \\
  Seoul National University \\
  \texttt{lightb0x@snu.ac.kr} \\
  \And
  Sungjoo Yoo \\
  Seoul National University \\
  \texttt{sungjoo.yoo@gmail.com} \\}

\begin{document}
\maketitle
\begin{abstract}
Ever-growing scale of large language models (LLMs) is pushing for improved efficiency, favoring fully quantized training (FQT) over BF16.
While FQT accelerates training, it faces consistency challenges and requires searching over an exponential number of cases, each needing over 200B tokens to ensure stability.

Pseudo-quantization training (PQT) addresses the issues of FQT, although it is not well-studied.
We explore the practical implications of PQT in detail and propose a noise distribution \(R\) that is floating-point (FP)-friendly, with ideal properties including stochastic precision annealing.
As a result, the proposed method serves as an effective theoretical foundation for low-precision FP parameters through PQT, utilizing efficient fake quantization via an addition and subsequent FP casting.

We demonstrate that Gaussian weight sampling is (1) scalable: supports low-precision FP parameters down to FP6 and high-precision noise up to 9-bit with BF16 operator.
The proposed method is (2) efficient: incurring computational overhead as low as 1.40\% on the A100 GPU in terms of Llama2 training tokens per second, and requiring 2 bytes per parameter in GPU memory.
We demonstrate that PQT with Gaussian weight sampling is (3) stable: closely following or even surpassing performance of the BF16 baseline while pre-training GPT2 and Llama2 models with up to 1B parameters and 300B tokens.
\end{abstract}

\begin{figure}[!ht]
\begin{subfigure}{0.56\columnwidth}
\includegraphics[trim={20 14 20 14},clip,width=\linewidth]{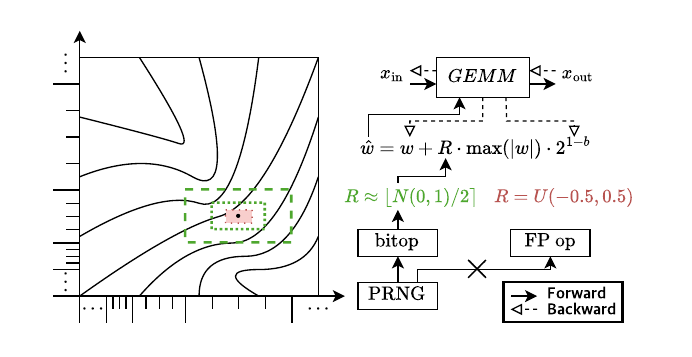}
\caption{\textbf{Overview of Gaussian weight sampling.} The figure on the left illustrates the loss landscape in real numbers, with FP-representable values on the axes.
The dot in the middle of the rectangle represents a parameter instance and the rectangles represent sampling space with the corresponding noise distribution on the right.
The figure on the right depicts computation graph of the proposed method.
We propose an FP-friendly noise distribution and an efficient noise generation method.}
\label{fig:overview_method}
\end{subfigure}
\hfill
\begin{subfigure}{0.42\columnwidth}
\includegraphics[trim={0 10 0 10},clip,width=\linewidth]{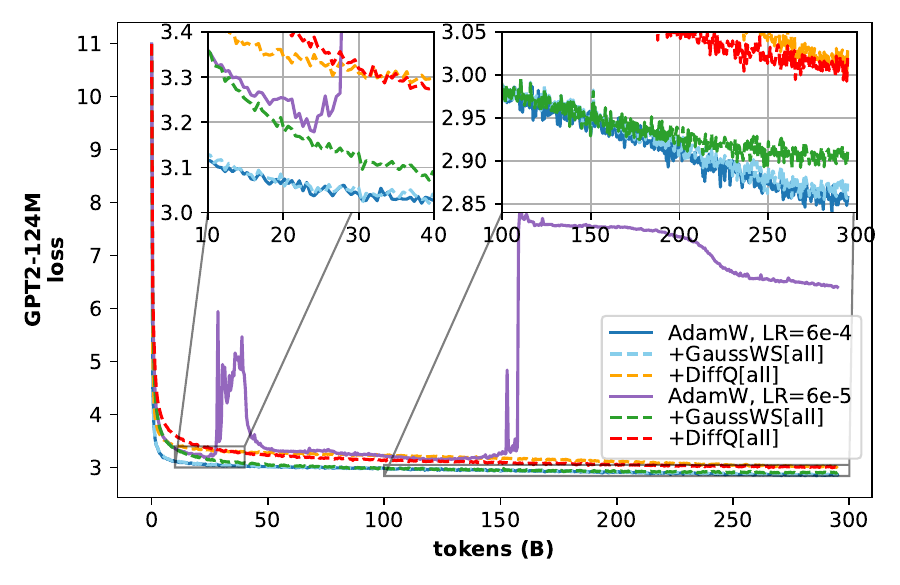}
\caption{\textbf{Loss curve of pre-training the GPT2-124M model on the OpenWebText dataset.}
``method[all]'' denotes that all linear layers of all transformer blocks adopt the corresponding method.
Both PQT methods mitigate training instability of baseline BF16 (purple).
GaussWS (skyblue and green) outperforms DiffQ (orange and red) throughout the training and rivals the best-case performance of baseline BF16 (blue).}
\label{fig:overview_result}
\end{subfigure}
\caption{Summary of Gaussian weight sampling (GaussWS).}
\label{fig:overview}
\end{figure}

\section{Introduction}
The training cost of large language model (LLM) has increased as the model size has grown over time.
For Llama 3 \cite{llama3}, the training utilizes 16k H100 GPUs, which require \(\approx\)11.2MW of electricity.
Studies have been conducted to reduce training cost. 
Parameter-efficient fine-tuning (PEFT) methods, \eg LoRA~\cite{LoRA}, QLoRA~\cite{QLoRA} and LoQT~\cite{LoQT}, reduce the size of training parameters.
FP8-LM~\cite{FP8_LM} reduces the volume of collective communication.
Adam-mini~\cite{adam_mini} and APOLLO~\cite{apollo} introduce optimizer with less internal state which results in less GPU memory required.

Especially, low-precision datatypes such as FP8~\citep{fp8fordl} and Microscaling (MX)~\citep{ocp_mx} have been proposed.
However, training with such datatypes faces two critical problems.
First, it faces consistency challenges, leading to suboptimal training, for example due to quantization-induced oscillations similar to traditional quantization-aware training (QAT) methods based on straight-through estimator~(STE) \citep{STE}\citep{chen2025oscillationreducedmxfp4trainingvision}.
Second, the problem of mapping parts of the model to specific bit precision is of an exponential complexity.
For example, mapping \(n\) linear layers to either FP8 or BF16 results in \(O(2^n)\) cases, each of which requires extensive training with over 200B tokens to validate training stability \citep{scaling_fp8}.
This makes manual search methods inefficient and suboptimal, thereby often resorting to simple, robust yet costly designs, \eg all BF16 training \citep{gemmateam2025gemma3technicalreport}.

Pseudo-quantization training (PQT), \eg DiffQ \cite{diffq} and NIPQ \cite{nipq}, can solve the problems.
PQT employs pseudo-quantization noise (PQN), which generalizes over actual quantization noise, during the training process to enable fully differentiable training.
It effectively reduces the search space to \(O(1)\) while mitigating the consistency challenges.
However, existing PQT methods do not thoroughly explore the numerical behavior of the computation, \ie addition and the subsequent casting.
In this paper, we show practical implications of PQT with respect to floating-point (FP) representation of model parameter \(w\).
Specifically, the addition of \(w\) and its corresponding PQN, followed by FP casting, induces underflow of the relatively smaller value among \(w\) and its PQN, effectively limiting the precision of the computation.

We propose rounded normal distribution as a basis of PQN to get an FP-friendly PQT solution which is precision-scalable, offers efficient implementations, and, importantly, enables stable training.
Firstly, the proposed PQN supports wide range of bitwidths. Assuming BF16 operator, rounded normal distribution enables PQN with effective 9-bit precision while uniform distribution supports up to 5-bit.
We also show that PQT with the proposed distribution conducts stochastic precision annealing, which renders the training robust to information loss at small-magnitude parameters \(w_{ij}\approx 0\).
This provides an upper bound on the dynamic range of \(w\), thereby determining the required number of FP exponent bit.
The proposed rounded normal distribution can be generated efficiently on both current and legacy hardware by leveraging bitwise operations. 
We demonstrate that the proposed method enables stable pre-training that closely follows, or even outperforms, the BF16 baseline for the GPT2-124M and the Llama2-\{134M, 1B\} models up to 300B tokens of training.

We compare the related methods in Table~\ref{tab:comparison}. To sum up, our contributions are as follows:
\begin{itemize}
\item We investigate practical implications of PQT that limit effective precision of FP.
\item We propose rounded Gaussian distribution as a basis of PQN which supports wide range of FP datatypes while enabling efficient noise generation.
\item We demonstrate that GaussWS enables stable PQT that closely follows, or even outperforms, the baseline BF16 on pre-training GPT2 and Llama2 language models up to 300B tokens.
\end{itemize}

\section{Preliminary}\label{sec:motivation}
\subsection{Fully quantized training and consistency challenges}
\label{sec:training_ocp_mx}
Reducing bit precision is one of the most effective methods to reduce the training cost of LLM \cite{FP8_LM}\cite{scaling_fp8}\cite{llm_fp4train}\cite{chen2025oscillationreducedmxfp4trainingvision}\cite{tseng2025trainingllmsmxfp4}\cite{deepseek_v3}\cite{ocp_mx}.
Specifically, as an illustration of MX matrix multiplication, consider multiplying two matrices \(A_{(M,K_Q)}\) and \(W_{(K_Q,N)}\) which are quantized along the inner dimension of the multiplication with block size of 32.
The subscript corresponds to the shape of the matrix and \(_Q\) denotes quantization axis.
The matrix multiplication is realized via vector dot products (of size-32 MX blocks) which are realized by MX-based compute units.

A na\"ive application of low-precision training can compromise the consistency of values between the forward and backward passes, which hurts training stability.
Consider forward and backward passes of an MX-compliant matrix multiplication where the quantization axis lies along the inner dimension:
\begin{gather}
    \label{eq:demo_forward}
    T_{(M,N)}=A_{(M,K_Q)}W_{(K_Q,N)}
    \\
    \frac{\partial L}{\partial W}_{(K,N)}
    =
    A^T_{(K, M_Q)}
    \frac{\partial L}{\partial T}_{(M_Q,N)}
    \quad\text{and}\quad
    \frac{\partial L}{\partial A}_{(M,K)}
    =
    \frac{\partial L}{\partial T}_{(M, N_Q)} W^T_{(N_Q, K)}
    \label{eq:demo_backward_2}
\end{gather}
where \(A\) denotes the input activation, \(W\) the parameter, \(T\) the output activation and \(L\) the target loss.
The subscript corresponds to the shape of the matrix and \(_Q\) denotes quantization axis.
Note the difference of \(W\) between the forward and backward passes, \ie \(W_{(K_Q,N)}\) compared to \(W^T_{(N_Q, K)}\), which is demonstrated in Figure~\ref{fig:mx_fb_error}.
This inconsistency can lead to suboptimal training~\citep{chen2025oscillationreducedmxfp4trainingvision}.
The issue arises because the absolute maximum value of the block, \eg size-2 blocks in Figure~\ref{fig:mx_fb_error}, changes when transposed.
Square-blockwise quantization can resolve this problem and ensure transpose-commutativity.

Another inconsistency, \ie oscillation problem, arises from the discrepancy between the high-precision parameter \(w\) and its low-precision quantized counterpart \(\hat{w}\).
The issue stems from updates \(\pm\epsilon\) on \(w\) becoming larger and biased when it comes to \(\hat{w}\) over multiple iterations of training \citep{diffq}\citep{nipq}.
To mitigate this inconsistency, stochastic rounding can be applied~\citep{tseng2025trainingllmsmxfp4}, and its noise can be generalized to pseudo-quantization noise~(PQN), leading to pseudo-quantization training~(PQT).

\subsection{Pseudo-quantization training}
\label{sec:problem_pqt}
Pseudo-quantization training (PQT) incorporates pseudo-quantization noise (PQN) that generalizes over actual quantization noise \(\hat{w}-w\) during training.
For example, the formulation of DiffQ is \(\hat{w}=w+\Delta\cdot U(-0.5,0.5)\), where \(\Delta\) is the stepsize \(\frac{\max(w)-\min(w)}{2^b-1}\) and \(b\) is a parameter representing the number of target bits for \(\hat{w}\).
Note that it can be seen as stochastically sampling \(\hat{w}\) from the space around \(w\).
The sampling is fully differentiable, allowing direct training of \(b\) without the bias of round-to-nearest and the approximation of STE.
As such, PQT can serve as a theoretical foundation for FQT.
However, current PQT methods lack practical consideration on the datatype of \(w\) and \(\hat{w}\), and we find them neither FP-friendly nor computationally efficient.

Current PQT methods are not FP-friendly, as they use uniform noise \(U(-0.5, 0.5)\) as a basis of PQN.
This requires unnecessary precision and disrupts forward-backward consistency with numerical instability.
Consequently, these methods are limited with respect to the range of ``safe'' bitwidths for PQN and necessitate high-precision operators such as FP32.
Refer to Section~\ref{sec:characteristics} for detail.

Current PQT methods are not computationally efficient.
They generate random values by performing FP operations on random integer streams produced by pseudo-random number generator (PRNG).
This exacerbates the bottleneck on vector operator (CUDA core) during the training process.
This issue is particularly pronounced on NVIDIA datacenter GPUs like the A100 because they have relatively lower vector operation throughput compared to their consumer counterparts.
Refer to Section~\ref{sec:generate} for detail and Section~\ref{sec:overhead_result} for empirical results.

\section{Method}\label{sec:gaussws}
\subsection{Overview}\label{sec:overview}
We aim to establish an effective theoretical foundation for the datatypes of \(w\) and \(\hat{w}\) by addressing the limitations of current PQT methods.
For the datatypes, we consider MX, particularly MXFP where internal datatype of MX is low-precision FP.
Firstly in Section~\ref{sec:gaussws_main}, we formulate Gaussian weight sampling that extends PQT to be MX-compliant.
Then in Section~\ref{sec:characteristics}, we study implications behind the choice of random distribution \(R\) on PQT and propose rounded Gaussian \(\lfloor \mathcal{N}(0,1)/2\rceil\).
Approximated rounded normal distributions can be generated efficiently by leveraging bitwise operations (Section~\ref{sec:generate}).
The kernels are implemented in Triton~\citep{triton} with decisions that favor predictably optimal throughput and straightforward implementation (Section~\ref{sec:kernel}).
Section~\ref{sec:impl_detail} describes implementation details, including the method to ensure unbiased PQN while maintaining forward-backward consistency, and the bitwidth parameter implementation.

\subsection{Gaussian weight sampling}
\label{sec:gaussws_main}
In Gaussian weight sampling, parameters are grouped into square block units to ensure transpose-commutativity as discussed in Section~\ref{sec:training_ocp_mx}.
Note that square-blockwise quantization is a special case of vector-wise quantization where adjacent vectors share the same scale, making it MX-compliant.

The formulation of Gaussian weight sampling is as follows:
\begin{equation}
    \label{eq:gws_forward}
    \hat{w}
    =w+R\odot \text{broadcast}_{b_l}\left(\max_{b_l}(|w|) \cdot 2^{1-b_t}\right)
\end{equation}
where \(w, \hat{w}, R\in\mathbb{R}^{m\times n}\), \(b_t\in\mathbb{R}^{\lceil m/b_l\rceil\times \lceil n/b_l\rceil}\), and \(b_l=32\) is the square block size following MX.
\(\max_{b_l}\) denotes square-blockwise maximum while \(\text{broadcast}_{b_l}\) is a function \(f:\mathbb{R}^{\lceil m/b_l\rceil\times\lceil n/b_l\rceil}\rightarrow \mathbb{R}^{m\times n}\) that replicates the same value square-blockwise.
\(\odot\) and \(\cdot\) denote the Hadamard product.
\(w\) is an original parameter, \(\hat{w}\) is a sampled parameter, \(R\) represents random and \(b_t\) is blockwise bitwidth.
We refer to the right-hand side of the addition as PQN.

Note that Equation~\ref{eq:gws_forward} is fully differentiable. With an approximation of \(\frac{\partial\max_{b_l}(|w|)}{\partial w}\approx 0\) assuming gradient to single element out of 32 by 32 block is negligible, we can calculate the gradient as follows:
\begin{equation}
\frac{\partial L}{\partial w}
=
\frac{\partial L}{\partial \hat{w}}
\quad\text{and}\quad
\frac{\partial L}{\partial b_t}
=
-\ln2\cdot \max_{b_l}(|w|)\cdot2^{1-b_t}\cdot\sum_{b_l}\left(\frac{\partial L}{\partial \hat{w}}\odot R\right)
\label{eq:gws_backward_2}
\end{equation}

\begin{figure}
    \centering
    \includegraphics[trim={0 10 0 6},clip,width=\linewidth]{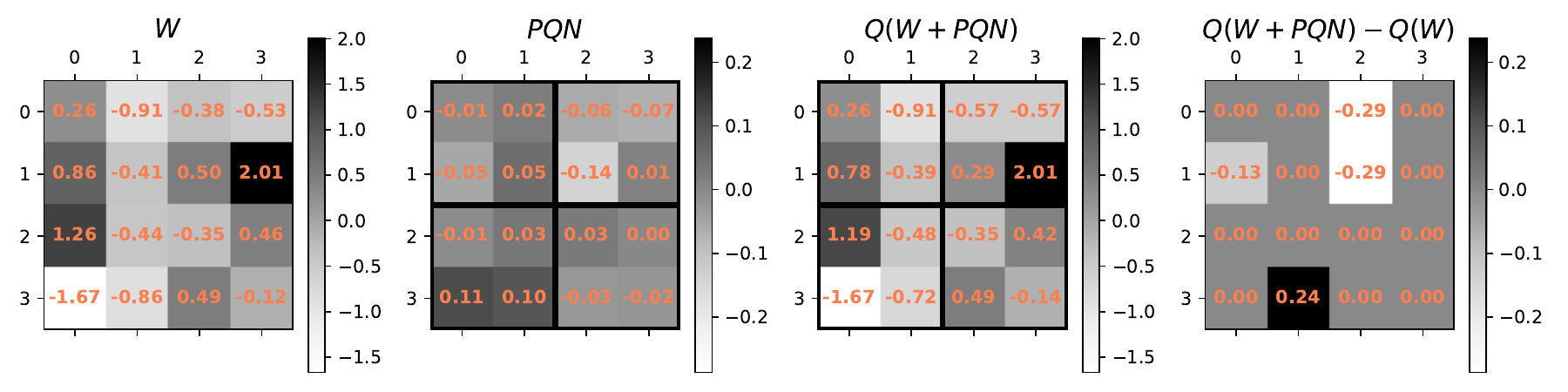}
    \caption{Example of effective PQN as in Equation~\ref{eq:gws_forward} with \(b_l=2\), \(R=U(-0.5,0.5)\) and \(b_t=4\) being underflow.
    Boxes wrapped in bold solid lines represent quantization groups with a square block size of 2, and an internal datatype of INT4 for \(Q(\cdot)\).
    The second matrix represents PQN reflected in the backward pass while the fourth matrix represents the effective PQN during the forward pass.}
    \label{fig:effective_PQN}
\end{figure}

\subsection{Effect of \textit{R} on PQT}
\label{sec:characteristics}
Note that in practice compute precision is limited and thus casting \(\hat{w}\), \eg to floating-point with \(e\)-bit exponent and \(m\)-bit mantissa, is required.
Here we discuss implications of such computation, \ie \(fp_{e,m}(\hat{w})\), during PQT.

FP casting \(fp_{e,m}(\cdot)\) introduces rounding errors similar to integer casting.
However, the backward pass in Equation~\ref{eq:gws_backward_2} has no indication of whether the PQN was rounded to zero or not in the forward pass of Equation~\ref{eq:gws_forward} and the subsequent FP casting.
See the demonstrated example in Figure~\ref{fig:effective_PQN}.
Therefore, the computation \(fp_{e,m}(\hat{w})\) should not underflow to ensure consistency of the values.
Specifically,
\begin{equation}
\label{eq:forward_backward}
fp_{e,m}(\hat{w}_{ij})
\neq
fp_{e,m}(w_{ij})
^\forall i,j|_{R_{ij}\neq 0}
\end{equation}
where the addition to get \(\hat{w}_{ij}\) limits precision of \(\min(|w_{ij}|,|PQN_{ij}|)\).
Thus, we consider the cases \(|w_{ij}|>|PQN_{ij}|\) and \(|w_{ij}|<|PQN_{ij}|\), using the smallest non-zero value of PQN (via \(2^\tau\triangleq \min_{R_{ij}\neq 0}|R|\)) and near-zero elements of \(w\) (\(\pm\epsilon\) where \(2^\xi\triangleq |\epsilon|\)).
Firstly, consider the former.
\begin{lemma}
\label{prop:precision_largest}
PQN that corresponds to \(b_t\)-bit do not underflow during the computation \(fp_{e,m}(\hat{w})\) if:
\begin{equation}
    \label{eq:underflow_bound}
    b_t < m + 2 + \tau
\end{equation}
\end{lemma}
\begin{proof}
    To prevent underflow of PQN, magnitude of the non-zero PQN should be larger than the stepsize of the FP exponent range corresponding to \(w\).
    Specifically, for each block and \(^\forall i,j|_{R_{ij}\neq 0}\),
    \begin{equation}
    \label{eq:underflow_bound_basis}
    \left|R_{ij}\right|\left(\max(|w|)\cdot 2^{1-b_t}\right)
    > 2^{\lfloor \log_2|w_{ij}|\rfloor-m}
    \end{equation}
    With \(2^{\tau}=\min_{R_{ij}\neq 0}|R|\) and \(\log_2\max|w|-\lfloor\log_2\max|w|\rfloor <1\), it simplifies to \(b_t < {m}+2+\tau\).
\end{proof}

\cref{prop:precision_largest} indicates that a larger \(\min_{R_{ij}\neq0}|R|\) is desirable with respect to PQT to maximize the flexibility of \(b_t\) given a fixed \(m\).
However, such \(R\) limits the precision of \(fp_{e,m}(\hat{w})\) in the latter case.

\begin{lemma}
\label{prop:precision_smallest}
The values of small magnitude \(\pm\epsilon\) in \(w\), where \(|\epsilon|=2^\xi\), do not underflow during the computation \(fp_{e,m}(\hat{w}_{ij})^\forall i,j|_{R_{ij}\neq 0}\) if:
\begin{equation}
\xi>\lfloor \tau+2-b_t+\log_2\max(|w|)\rfloor-m
\end{equation}
\end{lemma}

The proofs of \cref{prop:precision_smallest} and the following propositions are in Appendix~\ref{sec:proof}.
\cref{prop:precision_smallest} shows that there is a lower bound on the magnitude of \(\epsilon\) to prevent it from being 0.
This leads to an upper bound on effective number of exponent bits for the FP representation of \(w\) and thereby that of \(\hat{w}\).

\begin{proposition}
\label{prop:interpretation}
\textbf{(FP exponent cutoff)} Assuming \(\min_{w_{ij}\neq 0}(|w|)>2^{\lfloor \tau+2-b_t+\log_2\max(|w|)\rfloor-m}\), floating-point with a \(\lceil \log_2(-\tau+b_t+1)\rceil\)-bit exponent suffices to represent \(w\) without underflow, whereas floating-point with a \(\lceil \log_2(-\tau+b_t+3)\rceil\)-bit exponent suffices to represent \(\hat{w}\).
\end{proposition}

\cref{prop:precision_smallest} and \cref{prop:interpretation} show that smaller \(\min_{R_{ij}\neq0}|R|\) is desirable to maintain the precision of near-zero values of the original parameter \(w\) and its sampled counterpart \(\hat{w}\), which seems to counter \cref{prop:precision_largest}.
However, \(R\) with non-trivial \(Pr(R=0)\) is able to balance both cases and allows for larger \(\tau\) through the following:

\begin{proposition}
\label{prop:stochastic_precision_annealing}
\textbf{(Stochastic precision annealing)} Assume \(Pr(R=0)=p\). Computing \(\hat{w}\) preserves the precision of \(w\) with a probability of at least \(p\), while computing \(fp_{e,m}(\hat{w})\) masks \(|w_{ij}|<2^{\lfloor \tau+2-b_t+\log_2\max(|w|)\rfloor-m}\) with 0 at a probability of up to \(1-p\).
\end{proposition}

According to \cref{prop:stochastic_precision_annealing}, we can expect that PQT using non-trivial \(Pr(R=0)\) effectively trains the model to be robust to \(w\) with low dynamic range by stochastically annealing precision of \(w_{ij}\approx 0\).
It effectively enables PQT that targets low-precision FP in conjunction with \cref{prop:interpretation}.

To achieve both flexibility of \(b_t\) and stochastic precision annealing without significantly deviating from previous works~\citep{diffq}\citep{nipq}, we propose \(R=\lfloor \mathcal{N}(0,1)/2\rceil\).
This \(R\) allows for a wide range of bitwidths \(b_t\) according to \cref{prop:precision_largest}, thanks to a relatively larger \(\tau=0\).
For example with a BF16 operator, the rounded normal \(\lfloor \mathcal{N}(0,1)/2\rceil\) supports \(b_t<9\), whereas the range is narrower with \(b_t<5\) for a uniform \(U(-0.5,0.5)\) in 4-bit representation.
It also allows for low-precision representation of \(\hat{w}\) down to FP6 when \(b_t\le 4\), by stochastically annealing \(|w_{ij}|<2^{\lfloor 2-b_t+\log_2\max(|w|)\rfloor-m}\) with a probability \(\lessapprox 0.317\).
As a result, it renders lower bound of the datatypes as follows: FP with \(\lceil\log_2(b_t+1)\rceil\)-bit exponent for \(w\), and FP with \(\lceil\log_2(b_t+3)\rceil\)-bit exponent and \((b_t-2)\)-bit mantissa for \(\hat{w}\).
We obtain Table~\ref{tab:fp_bit} with different \(b_t\).
Lastly, the approximated \(R\approx\lfloor \mathcal{N}(0,1)/2\rceil\) can be generated efficiently on both current and legacy hardware by leveraging bitwise operators.

\subsection{Efficient generation of \textit{R}}
\label{sec:generate}
Pseudo-random number generators (PRNGs), \eg Philox \cite{philox} and Romu \cite{romu}, produce random bit stream, \ie random integers.
Random numbers in the real number domain are derived from these random integers.
\(U(0, 1)\) is derived by dividing the random integers by their maximum possible value.
Two samples of \(\mathcal{N}(0, 1)\) are derived from two samples from \(U(0, 1)\) using the Box-Muller transform \citep{box_muller}.

Note that generation of approximated \(R\approx\lfloor \mathcal{N}(0, 1)/2\rceil\) does not require aforementioned operations.
Assuming that each bit of the random integers from the PRNG is independently random, we create a random distribution that approximates the rounded normal distribution using two base cases:
\begin{equation}
\label{eq:prob_basis}
    \begin{cases}
P(A \land B) = P(A) \cdot P(B)\\
P(A \lor B) = P(A) + P(B) - P(A \land B)
    \end{cases}
\end{equation}
where \(A\) and \(B\) represent bitwise random variables, \(\land\) and \(\lor\) denote the ``and'' and ``or'' operators, respectively, and \(P(X)\) is shorthand for \(Pr(X=1)\).
Specifically, the distribution that we generate is:
\begin{equation}
\label{eq:approx_distr}
\begin{cases}
Pr(-2)=Pr(+2)=3/4\cdot 2^{-9} \approx 1/682.7 \\
Pr(-1)=Pr(+1)=(3/4)^2 \cdot 2^{-2}\cdot(1-Pr(\pm2))\approx 1/7.1 \\
Pr(0)=1-Pr(\pm1)-Pr(\pm2) \approx 0.717
\end{cases}
\end{equation}

In our implementation, the generated \(R\) values are represented in a sign-mantissa format with 4 bits per element, and 8 elements are packed into a 32-bit register.
Compared to 2's complement, the sign-mantissa format is simpler to generate and reconstruct into floating-point.

\subsection{Design decisions}\label{sec:kernel}
\textbf{Separate kernels.} While the BF16 baseline requires only one operation for the forward pass of linear layer, the GaussWS counterpart requires three operations: (1) generating \(R\), (2) unpacking \(R\) and adding scaled maximum, and (3) the matrix multiplication.
Fusing consecutive operations typically helps achieve maximum throughput by reducing GPU memory communication.
However, we decided not to fuse the operations, considering the following trade-offs.

Firstly, \(R\) generation is not fused.
PRNG is an algorithm that loops based on its internal state to generate random values iteratively.
The longer a PRNG's internal state is reused, the more it reduces the degree of parallelism, limiting the utilization of parallel hardware.
In other words, there exists an optimal ratio of parallelization to maximize throughput.
Furthermore, additional communication is required if the number of random values \(R\) generated and consumed per CUDA core does not match.
In practice, fusing the generation of \(R\) with the subsequent operations led to significant variation of throughput depending on the shape of \(w\).

Secondly, we do not fuse the scaled addition with the subsequent matrix multiplication.
This decision allows us to use the highly optimized PyTorch implementation of the linear operation while keeping implementation straightforward.

\textbf{GPU memory.} Equations~\ref{eq:demo_backward_2} and \ref{eq:gws_backward_2} show that the gradient of input activation in matrix multiplication requires \(\hat{w}\), while the gradient of \(b_t\) requires regenerating \(R\).
The same \(R\) can be generated using the same seed value used in the forward pass, temporarily requiring 0.5 bytes per parameter.

We explicitly store \(\hat{w}\) in BF16, although reconstructing \(\hat{w}\) in the backward computation would have reduced its GPU memory overhead, \ie 2 bytes per parameter.
This approach helps keep the implementation simple at the cost of a reasonable increase in GPU memory.
Note that the overhead of 2 bytes per parameter is negligible for small models and can be offset for larger models by leveraging, for example, training parallelism~\citep{rajbhandari2020zeromemoryoptimizationstraining}\citep{zhao2023pytorchfsdpexperiencesscaling}\citep{llama3}\citep{gemmateam2025gemma3technicalreport} and parameter-efficient optimizers~\citep{adam_mini}\citep{apollo}.

In conjunction, the design decisions described above enable a straightforward implementation where \(f(w,b_t)=\hat{w}\) is modularized into a single PyTorch module.

\subsection{Implementation details}
\label{sec:impl_detail}
\textbf{Managing seed.}
A seed value is required to initialize the PRNG, and here we discuss the specific requirements for it.
Firstly, the value of \(R\) in the forward pass must be identical to the value of \(R\) in the backward pass for proper training.
Additionally, to avoid bias across the entire model, the \(R\) values for each layer should be independently random.

To achieve these requirements, a multi-layer PRNG is employed to manage seeds and their corresponding random values.
First, a PRNG or seed generator is initialized with the user-specified seed value.
Second, the seed generator is used to produce seed values to initialize the PRNG of each layer.
Finally, the output of each layer's PRNG serves as the seed value for the GPU's PRNG, which then generates \(R\).
The state of each layer's PRNG is changed every gradient update during training.

\textbf{Bitwidth.}
We implemented an internal bitwidth parameter \(b_i\) for each 32 by 32 square unit of parameters in the linear layers.
\(b_i\) is linearly scaled to represent bitwidth \(b_t\) as follows:
\begin{equation}
    \label{eq:bitwidth}
    b_t=b_\text{target}+b_i\cdot(b_\text{init}-b_\text{target})
\end{equation}
where \(b_\text{init}\) and \(b_\text{target}\) are hyperparameters representing the initial and target bitwidths, respectively.
\(b_i\) should be initialized with 1.
\(b_t\) is guided towards \(b_\text{target}\) through the weight decay applied to \(b_i\).

A loss term related to \(b_t\) can also be incorporated into the training loss \(L\):
\begin{equation}
    \label{eq:bitwidth_loss}
    L^\prime=L+\lambda\sum_{i=1}^{n} \frac{\sum_{j=1}^{m_i}|b_{t}^{i,j}-b_\text{target}|}{m_i}
\end{equation}
where \(n\) is the number of layers, \(m_i\) is number of square blocks in \(i\)-th layer and \(b_t^{i,j}\) denotes bitwidth of \(i\)-th layer and \(j\)-th block.
In this scenario, an additional hyperparameter \(\lambda\) is required to appropriately scale the loss associated with the bitwidth parameter.

\section{Experimental results}\label{sec:result}
Transformer~\cite{vaswani2023attentionneed}-based language models were trained from scratch: the GPT2-124M model~\citep{radford2019language} on the OpenWebText dataset~\citep{Gokaslan2019OpenWeb}, and the Llama2-134M and Llama2-1B models~\citep{touvron2023llama2openfoundation} on the C4 dataset~\citep{raffel2023exploringlimitstransferlearning}.
The loss curve of pre-training and the resulting bitwidth \(b_t\) are presented in Section~\ref{sec:opt_owt}.
Section~\ref{sec:overhead_result} reports the overhead of the proposed method.

We use ``method[part]'' to represent which linear layer(s) of all transformer blocks adopt the corresponding method.
[od] is used as shorthand for [out,down].
Note that the GPT2 transformer block comprises four linear layers: qkv, out, up, and down.
The qkv and out layers, along with the self-attention operation, constitute the attention module, while the up and down layers form the feed-forward module.
``DiffQ'' represents an extension of DiffQ~\citep{diffq}, which is equivalent to GaussWS except for BF16 \(U(-0.5, 0.5)\) in place of \(\approx\lfloor \mathcal{N}(0,1)/2\rceil\).

We used BF16 GEMM with FP32 accumulation.
Current FP datatypes with fewer than 16 bits, \eg FP8, support the lower bound on the exponent and mantissa bits of model parameters only when \(b_t\le 5\).
See the requirements of the proposed method in Table~\ref{tab:fp_bit} and the bitwidth results in Figure~\ref{fig:bitwidth_result}.
We used \(b_\text{init}=6\) and \(b_\text{target}=4\) unless otherwise specified.
Refer to Appendix~\ref{sec:experiment_detail} for detailed settings.

\begin{figure*}[ht]
    \begin{subfigure}{0.5\linewidth}
        \includegraphics[trim={0 12 0 10},clip,width=\linewidth]{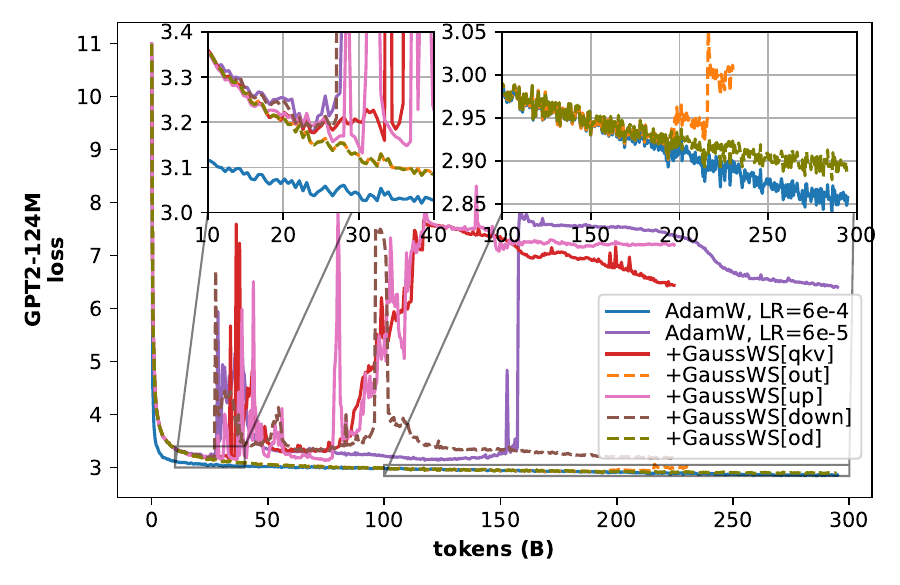}
        \caption{Stability case study with AdamW and \(b_\text{target}=0\).}
        \label{fig:opt2_adamw}
    \end{subfigure}
    \begin{subfigure}{0.5\linewidth}
        \includegraphics[trim={0 12 0 10},clip,width=\linewidth]{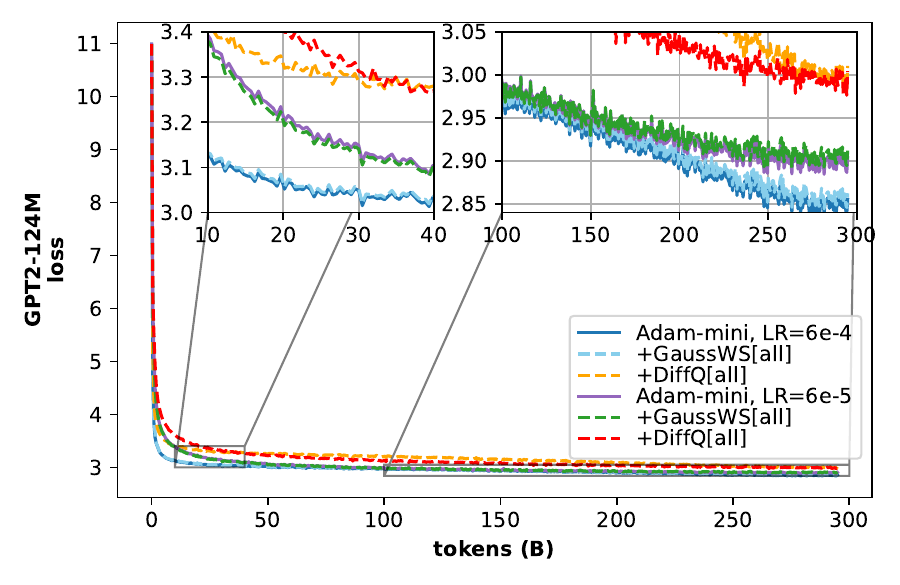}
        \caption{With Adam-mini optimizer.}
        \label{fig:opt2_adam_mini}
    \end{subfigure}
    \caption{Training loss curve of the GPT2-124M model on the OpenWebText dataset.}
    \label{fig:opt2_training}
\end{figure*}

\subsection{Pre-train results}\label{sec:opt_owt}
\textbf{The GPT2-124M model} is trained from scratch on the OpenWebText dataset up to 300B tokens~\citep{Karpathy2022}.
Results in Figure~\ref{fig:overview_result} show that the baseline BF16 training with a learning rate of \(6\times 10^{-4}\) proceeds smoothly whereas the counterpart with a learning rate of \(6\times 10^{-5}\) diverges and fails to recover.
Both PQT methods mitigate such training instability while the proposed method incurs minimal increase in loss.
The difference in performance between GaussWS and DiffQ is attributed to the choice of \(R\).
GaussWS consistently outperforms DiffQ, which aligns with the properties in Section~\ref{sec:characteristics}, especially considering bitwidth result of Figure~\ref{fig:bitwidth_result}.

To identify the source of baseline BF16 training instability, we restrict the application of the proposed method to each of the linear layers within all transformer blocks.
As shown in Figure~\ref{fig:opt2_adamw}, GaussWS[qkv], GaussWS[up], and GaussWS[down] begin to diverge at \(\approx\)30B tokens of training and fail to recover.
In contrast, GaussWS[out] does not diverge and closely approximates the best-case loss curve of baseline BF16 up to \(\approx\)200B tokens.
GaussWS[od], which applies the proposed method to the last layers of the residual addition branches in the transformer blocks, reduces divergence and yields the best result with the smaller learning rate.
These results show that the attention module is the source of instability at \(\approx\)30B tokens of training, while the feed-forward module is the source of instability at \(\approx\)200B tokens of training. The latter is consistent with \citet{scaling_fp8}.

As an example of parameter-efficient optimizers, we report pre-training results with Adam-mini~\citep{adam_mini} in Figure~\ref{fig:opt2_adam_mini}.
The Adam-mini optimizer stabilizes BF16 baseline training and slightly improves DiffQ compared to AdamW~\citep{loshchilov2019decoupledweightdecayregularization}, while GaussWS is orthogonal to the choice of optimizer.

\begin{figure*}[ht]
    \includegraphics[trim={0 12 0 12},clip,width=\linewidth]{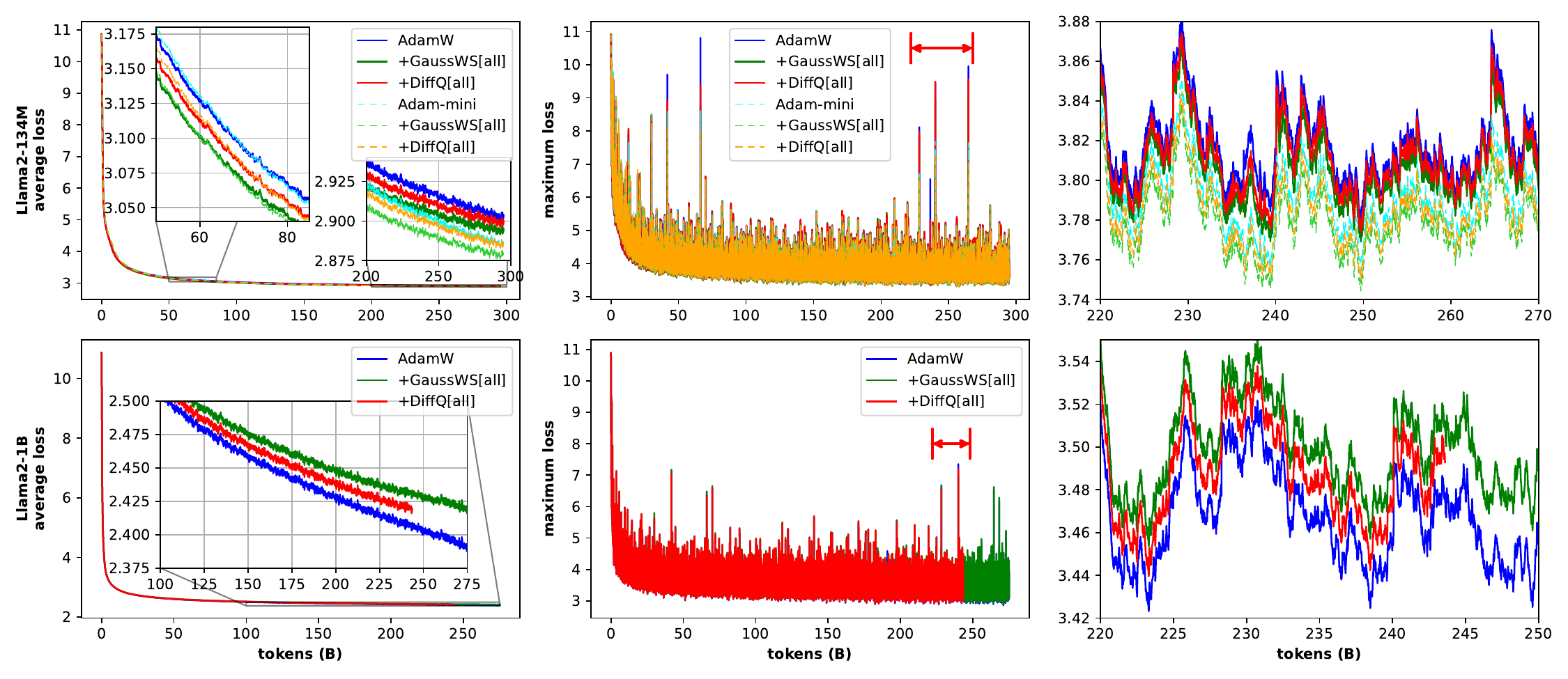}
    \caption{Training loss curve of the Llama2-\{134M, 1B\} models on the C4 dataset. First column represents average loss and the other two represent maximum loss.
    Third column corresponds to the range annotated with the orange arrow on the second column.
    For better visualization, weighted moving average is used with \(\alpha=1/16\) on left column and \(\alpha=1/128\) on right column.}
    \label{fig:llama2_training}
\end{figure*}

\textbf{The Llama2-134M and Llama2-1B models} are trained from scratch on the C4 dataset up to 300B and 275B tokens, respectively~\citep{torchtitan}.

As shown on the first row of Figure~\ref{fig:llama2_training}, GaussWS improves Llama2-134M pre-training for both average and worst case.
Results with GaussWS require fewer tokens for Adam-mini to surpass AdamW.
DiffQ lies in between baseline BF16 and GaussWS unlike GPT2 results.

On the other hand, the second row of Figure~\ref{fig:llama2_training} show that Gaussian weight sampling slightly degrades Llama2-1B pre-training for both average and worst case.
We conjecture that the increase in loss is attributed to the lower bitwidth \(b_t\) as shown in Figure~\ref{fig:bitwidth_result}.
Note that optimal precision \(b_\text{opt}\) of larger, over-trained models tend to be larger~\cite{kumar2024scalinglawsprecision}\cite{sun2025scalinglawsfloatingpoint}.
It implies that the results can be improved by tuning hyperparameters \(b_\text{init}\), \(b_\text{target}\), weight decay on \(b_i\) and an optional \(\lambda\), as in Figure~\ref{fig:larger_hyperparam}.

\begin{figure}
  \includegraphics[trim={0 12 0 12},clip,width=\linewidth]{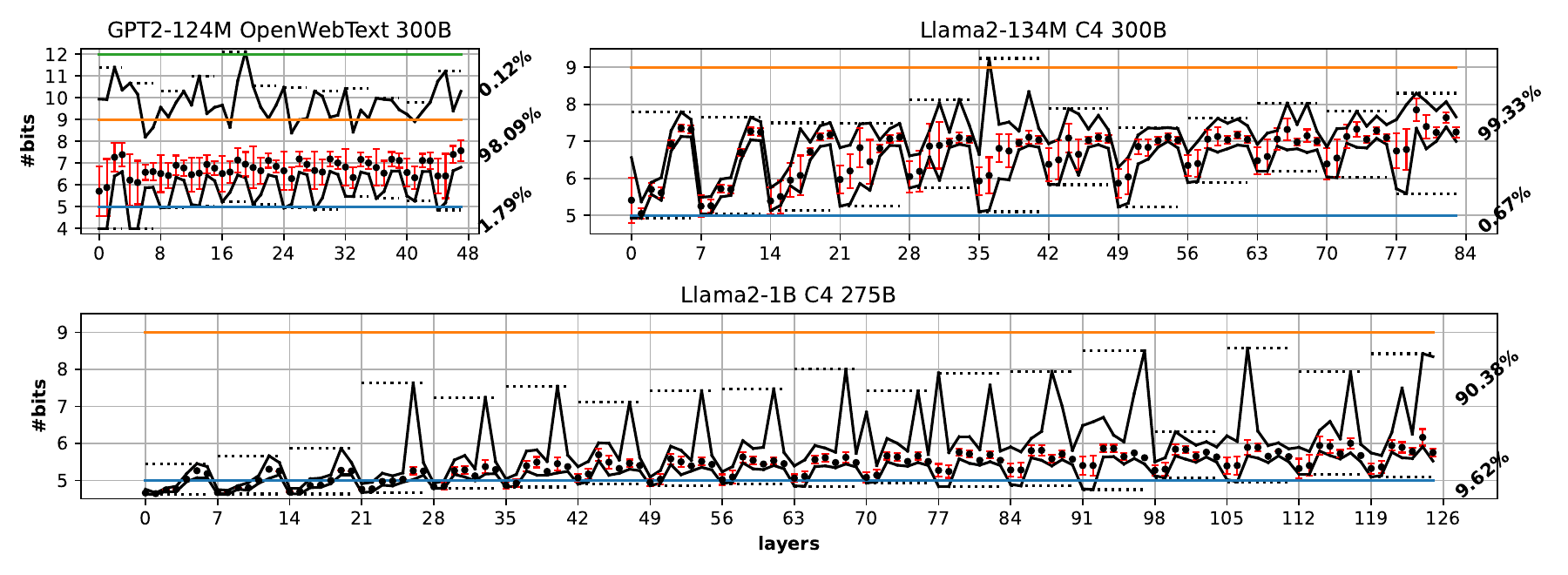}\hfill
  \caption{Resulting bitwidth \(b_t\).
  Dots and red lines indicate layerwise mean and standard deviation.
  Upper and lower solid lines represent layerwise maximum and minimum while dotted lines represent transformer-blockwise maximum and minimum.
  Lines on 5, 9 and 12-bit divide the parameters into 3 tiers, and the percentages on the right-hand side represent the ratio of parameters for each.
  The order of layers is (qkv, out, up, down) for GPT2 and (q, k, v, out, gate, down, up) for Llama2.}
  \label{fig:bitwidth_result}
\end{figure}

\textbf{Resulting bitwidth \(b_t\)} is visualized in Figure~\ref{fig:bitwidth_result}.
The results imply that GPT2-style transformer blocks require a greater dynamic range for the parameters compared to their Llama2-style counterparts.
More than 99\% of the parameters are robust to PQN with \(b_t\le 9\), irrespective of the architecture.

\begin{table}
\tabcolsep=0.11cm
    \centering
    \caption{Tokens per second per GPU (left) and GPU memory usage (right) during Llama2 pre-training on the A100-SXM4-40G GPU.
    Subscript denotes relative overhead compared to BF16 baseline.
    We used local batch size \{24, 8, 2, 2\} respectively for each case of \{134M, 1B, 3B, 70B\(^\dagger\)\} with fixed sequence length of 2048.
    ``\(^\dagger\)''~denotes that only 4 layers out of total 80 layers of the model are used.}
    \begin{tabular}{l cccc c cccc}
    \hline
       & \multicolumn{4}{c}{tps per GPU (\(\times 10^3\))} &&\multicolumn{4}{c}{GPU memory (GiB)} \\
    \cline{2-5} \cline{7-10}
        & \textbf{134M} & \textbf{1B} & \textbf{3B} & \textbf{70B\(^\dagger\)} && \textbf{134M} & \textbf{1B} & \textbf{3B} & \textbf{70B\(^\dagger\)} \\
   \hline
        AdamW & 143.3 & 26.0 & 7.17 & 7.22 && 34.00 & 30.69 & 19.07 & 18.83 \\
        +GaussWS[all] & 141.3\(_{1.40\%}\) & 25.5\(_{1.92\%}\) & 6.79\(_{5.30\%}\) & 6.94\(_{3.88\%}\) && 34.16 & 32.42 & 24.99 & 23.42 \\
        +DiffQ[all] & 116.6\(_{18.63\%}\) & 23.1\(_{11.15\%}\) & 5.00\(_{30.26\%}\) & 5.21\(_{27.84\%}\) && 34.18 & 32.64 & 25.76 & 25.57 \\
    \hline
        Adam-mini & 93.9 & 21.1 & 5.41 & 7.25 && 33.87 & 29.70 & 17.50 & 17.00  \\
        +GaussWS[all] & 85.7\(_{8.73\%}\) & 19.3\(_{8.53\%}\) & 4.65\(_{13.97\%}\) & 6.97\(_{3.86\%}\) && 34.03 & 31.43 & 23.43 & 21.58 \\
        +DiffQ[all] & 82.3\(_{12.35\%}\) & 17.8\(_{15.64\%}\) & 3.71\(_{31.32\%}\) & 5.19\(_{28.41\%}\) && 34.05 & 31.65 & 24.19 & 23.73 \\
    \hline
    \end{tabular}
    \label{tab:overhead_result}
\end{table}

\begin{figure}
   \includegraphics[trim={0 12 0 12},clip,width=\linewidth]{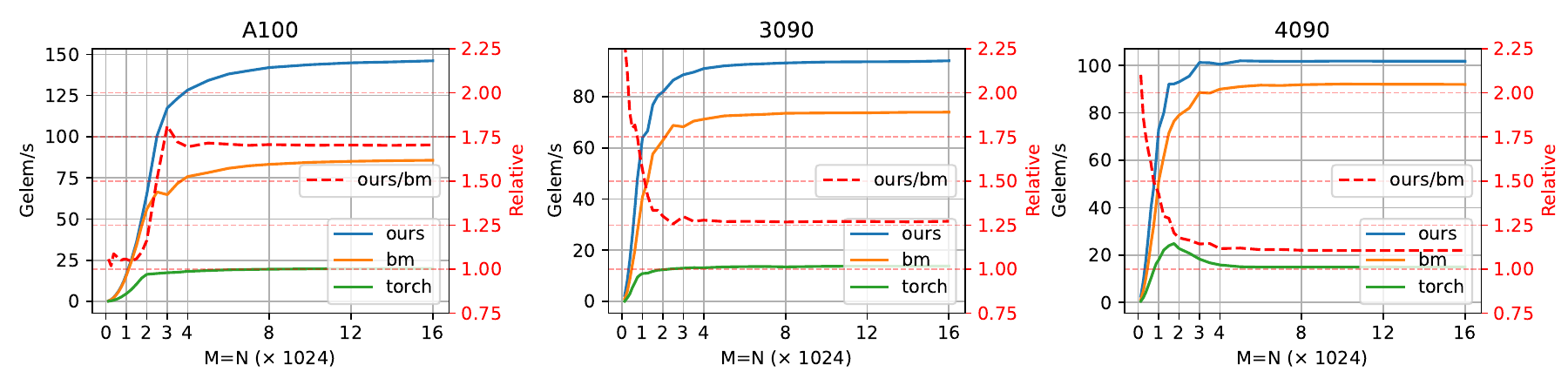}\hfill
  \captionof{figure}{Forward pass benchmark results for the PyTorch layer implementing Equation~\ref{eq:gws_forward} on a matrix \(w_{(M, N)}\).
  Absolute throughput in \(10^9\) elements per second.
  ``torch'' indicates PyTorch baseline while the other two are implemented in Triton.
  ``bm'' implements Box-Muller transform and ``ours'' implements the proposed generation method described in Section~\ref{sec:generate}.}
  \label{fig:kernel}
\end{figure}

\subsection{Overhead}\label{sec:overhead_result}
Table~\ref{tab:overhead_result} reports the throughput and GPU memory usage during the Llama2 model training.
The proposed generation method minimizes \textbf{computational overhead}.
The geometric mean of the overhead on training throughput for Llama2-\{134M, 1B, 3B, 70B\(^\dagger\)\} with AdamW is 3.14\% and 22.34\% for GaussWS and DiffQ, respectively, while it is 8.84\% and 22.35\% with Adam-mini.

On the other hand, \textbf{GPU memory overhead} is 2 bytes per parameter to store \(\hat{w}\) in BF16.
The proposed method requires less layerwise temporary memory to store \(R\), using 0.5 bytes per element for \(\approx\lfloor \mathcal{N}(0,1)/2\rceil\) compared to 2 bytes for \(U(-0.5, 0.5)\).

Figure~\ref{fig:kernel} presents the results of unit benchmark for the forward pass of the proposed method.
Both the proposed method and the Box-Muller method demonstrate at least a \(3\times\) improvement compared to the PyTorch baseline, as they are implemented in Triton and reduce global memory communication.
The proposed noise generation method enhances throughput compared to the Box-Muller method across all test cases.
It is particularly effective with larger matrix and the A100 GPU.
Note that weight dimension of Llama 3.2 1B ranges from (2k, 0.5k) to (2k, 8k) while Llama 3.1 405B ranges from (16k, 1k) to (16k, 16k).

\section{Discussion, broader impact and limitation}\label{sec:discussion}
The large discrepancy in loss between GaussWS and DiffQ in the GPT2-124M model pre-training, as opposed to the Llama2 models, implies that the proposed \(R\) is effective in handling the larger dynamic range of GPT2 parameters via stochastic precision annealing.

The stable pre-training results achieved with the proposed method suggest that the low-precision FP datatypes in Table~\ref{tab:fp_bit} are sufficient.
Specifically, the bitwidth results in Figure~\ref{fig:bitwidth_result} demonstrate the compatibility of the parameters with \(b_t\le 5\), \(\le 9\) and \(\le 12\), which are supported by the datatypes FP8\_e3m4, FP12\_e4m7 and FP16\_e5m10, respectively.

GaussWS is comparable to FQT, which is more efficient but does not guarantee precision scalability and training stability.
We believe this tradeoff is meaningful considering the cost of extensive training required to check stability and the limited throughput improvement of FQT, \eg \(\approx 34\%\) over BF16 \citep{scaling_fp8}.
In conjunction, GaussWS can serve as a cornerstone for the next generation of unified quantization frameworks and standardized datatypes for machine learning.

The proposed method is applied only on weight, leaving activation and gradient same as baseline BF16.
In particular, it is impossible to conduct differentiable search on gradient.
Extending the proposed method to activation is left as future work.

\begin{ack}
This was supported by Mobile eXperience(MX) Business, Samsung Electronics Co., Ltd.
\end{ack}

\bibliography{main}

\appendix
\renewcommand\thefigure{\thesection.\arabic{figure}}
\renewcommand\thetable{\thesection.\arabic{table}}

\pagebreak

\setcounter{figure}{0}
\setcounter{table}{0}
\section{Proof}
\label{sec:proof}
\textbf{\cref{prop:precision_smallest}.}
\begin{proof}
Consider adding PQN to \(\{0, \pm\epsilon\}\) as following:
\begin{equation}
    \{0, \pm\epsilon\} + R\odot \text{broadcast}_{b_l}\left(\max_{b_l}(|w|)\cdot 2^{1-b_t}\right)
\end{equation}
Without loss of generality, consider a single block and the smallest positive perturbation \(2^{\tau}=\min_{R_{ij}\neq0}|R|\) that limits the precision of floating-point \(\pm\epsilon\) in the least.
The value of interest is:
\begin{equation}
    \{0, \pm\epsilon\}+2^{\tau+1-b_t}\max(|w|)
\end{equation}
Note that \(2^{\tau+1-b_t}\max(|w|)\) lies in the range \([2^{\lfloor s\rfloor},2^{\lfloor s\rfloor+1})\) where \(s\triangleq \tau+1-b_t+\log_2\max(|w|)\).
Assuming \(0<\epsilon<2^{\tau+1-b_t}\max(|w|)\), then 
\(\pm\epsilon+2^{\tau+1-b_t}\max(|w|)\) lies in the range \([0,2^{\lfloor s\rfloor+2})\), making the stepsize at most \(2^{\lfloor s\rfloor+1-m}\).
Therefore, the condition that suffices to prevent underflow is:
\begin{equation}
    \xi>\lfloor s\rfloor+1-m=\lfloor \tau+2-b_t+\log_2\max(|w|)\rfloor-m
\end{equation}
\end{proof}

\textbf{\cref{prop:interpretation}.}
\begin{proof}
First, we consider dynamic range of \(w\) via stepsize of the smallest positive and the largest values in floating-point representation of \(|w|\).
The given condition is equivalent to the following, as in \cref{prop:precision_smallest}:
\begin{equation}
|\epsilon|>2^{\lfloor \tau+2-b_t+\log_2\max(|w|)\rfloor-m}
\end{equation}
Note that \(2^{\lfloor \log_2\max(|w|)\rfloor-m}\) is the largest stepsize that corresponds to the widest exponent range when representing \(w\) in FP.
We can count the number of FP exponent ranges from the largest to the smallest, to get \((-\tau-1+b_t)\) effective normal ranges.
FP with \(\lceil\log_2(-\tau+b_t+1)\rceil\)-bit exponent suffices to support all effective exponent ranges of \(w\), encompassing normal ranges, a subnormal range and a range for NaN/Inf.

Then, we consider the dynamic range of \(\hat{w}=w+R\cdot 2^{1-b_t}\max(|w|)\) through the smallest non-zero and the largest possible values of \(|\hat{w}|\).
For lower bound, without loss of generality, consider \(w>0\) and \(|w-2^{\tau+1-b_t}\max(|w|)|\).
Non-zero lower bound of \(|\hat{w}|\) can be derived by a single stepsize in the exponent range that \(\max\left(|w_{ij}|,2^{\tau+1-b_t}\max(|w|)\right)\) reside.
The problem of interest is \(\min_{i,j}\max\left(|w_{ij}|,2^{\tau+1-b_t}\max(|w|)\right)\), and the solution lies in the following range:
\begin{equation}
\label{eq:hat_w_lower_bound}
   [2^{\lfloor\tau+1-b_t+\log_2\max(|w|)\rfloor},2^{\lfloor\tau+2-b_t+\log_2\max(|w|)\rfloor})
\end{equation}
On the other hand for upper bound, assume \(0<|R\cdot 2^{1-b_t}\max(|w|)|<\max(|w|)\).
The widest FP exponent range that \(\hat{w}=w+R\cdot 2^{1-b_t}\max(|w|)\) can fall into is:
\begin{equation}
\label{eq:hat_w_upper_bound}
    [2^{\lfloor \log_2\max(|w|)\rfloor+1},2^{\lfloor\log_2\max(|w|)\rfloor+2})
\end{equation}
which is off-by-one from the range \([2^{\lfloor \log_2\max(|w|)\rfloor},2^{\lfloor \log_2\max(|w|)\rfloor+1})\) of \(w\).
There are \((-\tau+1+b_t)\) normal exponent ranges, from the upper bound in Equation~\ref{eq:hat_w_upper_bound} to the lower bound in Equation~\ref{eq:hat_w_lower_bound} (both inclusive).
FP with \(\lceil\log_2(-\tau+3+b_t)\rceil\)-bit exponent suffices to support the normal ranges, a subnormal range and a range for NaN/Inf.
\end{proof}

\textbf{\cref{prop:stochastic_precision_annealing}.}
\begin{proof}
As shown in the proof of \cref{prop:precision_smallest}, values of small magnitude \(\pm\epsilon\) in \(w\), where \(|\epsilon|=2^\xi\) and \(\xi<\lfloor \tau+2-b_t+\log_2\max(|w|)\rfloor-m\), can underflow and effectively become \(\epsilon=0\) during the computation of \(fp_{e,m}(\hat{w})\) with \(Pr(R\neq 0)= 1-p\).
On the other hand, \(R_{ij}=0\) preserves the precision of \(w_{ij}\) with \(Pr(R=0)=p\).
\end{proof}

\setcounter{figure}{0}
\setcounter{table}{0}
\section{Comparison of related methods}
\label{sec:comparison}

Table~\ref{tab:comparison} compares the important properties of related methods: BF16 baseline, FQT and PQT (DiffQ, NIPQ and GaussWS).
Notably, GaussWS achieves the best stability and flexibility out of all methods without compromising accuracy and throughput.
NIPQ was not tested because of its larger overhead.
Clamping increases computational overhead and requires additional GPU memory to store the corresponding mask.
While DiffQ and NIPQ aim for integer quantization, GaussWS targets both floating-point and integer quantization.

\begin{table}
    \centering
    \caption{Comparison of related methods. ``-'' denotes a case that is not tested due to its overhead.}
    \begin{tabular}{c|cccc}
         & Throughput & Stability & Accuracy & Flexibility \\
         \hline
         BF16 & Good & Good & Best & No \\
         FQT & Best & No guarantee & No guarantee & No \\
         DiffQ~\citep{diffq} & Worse & Best & Good & Good \\
         NIPQ~\citep{nipq} & Worst & - & - & Good \\
         GaussWS & Good & Best & Best & Best \\
    \end{tabular}
    \label{tab:comparison}
\end{table}

\setcounter{figure}{0}
\setcounter{table}{0}
\section{Floating-point datatypes for resulting bitwidth}
\label{sec:fp_datatype}
Table~\ref{tab:fp_bit} reports the FP datatypes discussed in Section~\ref{sec:characteristics} with the proposed \(R\) and wide range of \(b_t\).
Note that we did not cover mantissa of \(w\).
The number of mantissa bits for \(w\) should be larger than that of \(\hat{w}\) to fully utilize the datatype of \(\hat{w}\).
Additionally, it should account for the dynamics of the update on \(w\), \eg \(\gamma\cdot m_t/(\sqrt{v_t}+\epsilon)\) for AdamW.

\begin{table}
  \centering
  \captionof{table}{Floating-point datatypes that are discussed in Section~\ref{sec:characteristics} with respect to \(b_t\) of the proposed method.
  ``Datatype \(\hat{w}\)'' denotes current de-facto standard FP datatypes that support the distribution of \(\hat{w}\).
  FP\texttt{n}\_e\texttt{E}m\texttt{M} represents \texttt{n}-bit floating-point with \texttt{E}-bit exponent and \texttt{M}-bit mantissa.
  }
  \begin{tabular}{ccccc}
    \hline
    \textbf{\(b_t\)} & \textbf{Exponent \(w\)} & \textbf{\(e\) (Exponent \(\hat{w}\))} & \textbf{\(m\) (Mantissa \(\hat{w}\))} & \textbf{Datatype \(\hat{w}\)} \\
    \hline
    3 & 2 & 3 & 1 & FP6\_e3m2 \\
    4 & 3 & 3 & 2 & FP6\_e3m2 \\
    5 & 3 & 3 & 3 & FP8\_e4m3, FP8\_e3m4 \\
    6 & 3 & 4 & 4 & BF16, FP16 \\
    7 & 3 & 4 & 5 & BF16, FP16 \\
    8 & 4 & 4 & 6 & BF16, FP16 \\
    9 & 4 & 4 & 7 & BF16, FP16 \\
    10 & 4 & 4 & 8 & FP16 \\
    11 & 4 & 4 & 9 & FP16 \\
    12 & 4 & 4 & 10 & FP16 \\
    13 & 4 & 4 & 11 & FP32 \\
    \hline
  \end{tabular}
  \label{tab:fp_bit}
\end{table}

\setcounter{figure}{0}
\setcounter{table}{0}
\section{Visualized example of forward-backward inconsistency}
\label{sec:example_forward_backward_inconsistency}
Figure~\ref{fig:mx_fb_error} visualizes forward-backward error in vector-wise quantization with the quantization axis only on the inner dimension of the matrix multiplication.
The second matrix represents the value reflected in the backward pass while the third matrix represents its forward pass counterpart.

\begin{figure}
    \centering
    \includegraphics[trim={0 12 0 6},clip,width=\linewidth]{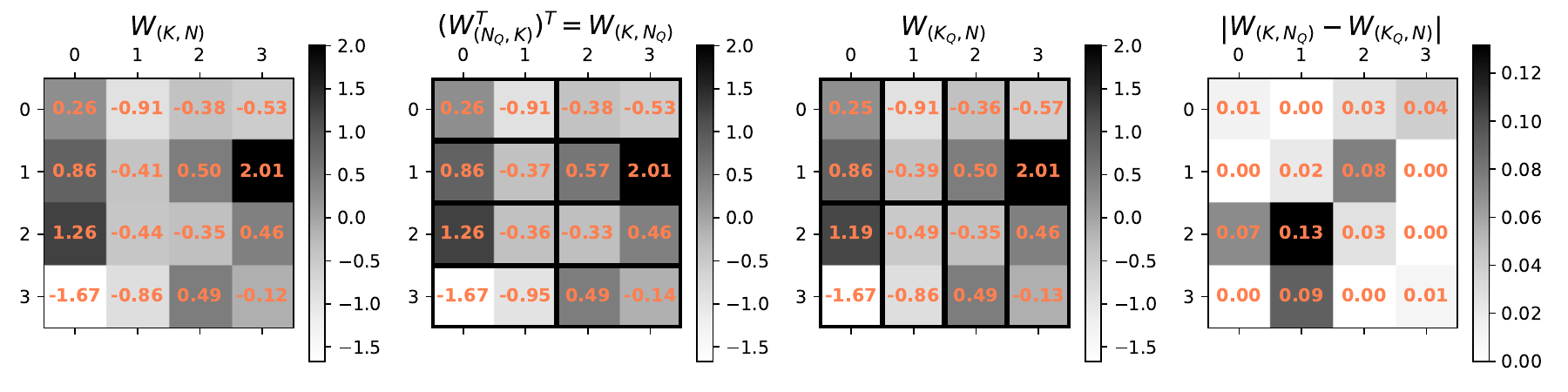}
    \caption{An example of vector-wise quantization on \(W_{(K,N)}\sim \mathcal{N}(0,1)\) and its forward-backward discrepancy, where \(K=N=4\).
    Boxes wrapped in bold solid lines represent quantization groups with an internal datatype of INT4 and a block size of 2.
    Visualized values are fake-quantized.}
    \label{fig:mx_fb_error}
\end{figure}

\setcounter{figure}{0}
\setcounter{table}{0}
\section{Pre-train hyperparameter, setup and resource}
\label{sec:experiment_detail}
Table~\ref{tab:hyperparam} reports hyperparameters used for pre-training. Learning rate was linearly scheduled with warmup.
The Llama2-1B model training requires more than 24GiB of GPU memory.

\begin{table}
\tabcolsep=0.11cm
  \centering
  \caption{Hyperparameters used for pre-training.} %
  \begin{tabular}{cccc}
    \hline
     & \textbf{GPT2-124M} & \textbf{Llama2-134M} & \textbf{Llama2-1B} \\
    \hline
    context window & 1024 & 2048 & 2048 \\
    local batch & 12 & 12 & 8 \\
    local grad. accum. per step & 5 & 1 & 1 \\
    \# GPUs & 8 & 8 & 8 \\
    total steps & 600k & 1.5M & 2.1M \\
    warmup steps & 2k & 5k & 7k \\
    total number of training tokens (\(\times 10^6\)) & 294,912 & 294,912 & 275,251.2 \\
    min learning rate & \{6e-5, 6e-6\} & 1e-6 & 1e-6 \\
    max learning rate & \{6e-4, 6e-5\} & 1e-5 & 1e-5 \\
    weight decay & 0.1 & 0.1 & 0.1 \\
    \(\lambda\) (as in Equation~\ref{eq:bitwidth_loss})  & 1e-4 & 0 & 0 \\
    \hline
  \end{tabular}
  \label{tab:hyperparam}
\end{table}

For GPT2, we used \citet{Karpathy2022} with commit \texttt{9755682b} as a starting point and \texttt{nvcr.io/nvidia/pytorch:24.10-py3} as a training environment.
For Llama2, we used \citet{torchtitan} with commit \texttt{90567fc9} as a starting point and \texttt{ghcr.io/pytorch/pytorch-nightly} with a tag \texttt{2.7.0.dev20250107-cuda12.4-cudnn9-devel} as a training environment.
We used A100-SXM4-40G, RTX 3090 and RTX 4090 GPUs with NVIDIA driver R565.

\setcounter{figure}{0}
\setcounter{table}{0}
\section{Preliminary result of Llama2-1B with higher bitwidth}
\label{sec:larger_hyperparam}
Figure~\ref{fig:larger_hyperparam} reports the preliminary result of pre-training the Llama2-1B model with \(b_\text{init}=8\) and \(b_\text{target}=6\).
Although preliminary, GaussWS with \(b_\text{init}=8\) and \(b_\text{target}=6\) is comparable to BF16 baseline.
The results for GaussWS and DiffQ will be prepared before July 2025.

\begin{figure*}[ht]
    \begin{subfigure}{0.33\linewidth}
        \includegraphics[trim={10 8 40 20},clip,width=\linewidth]{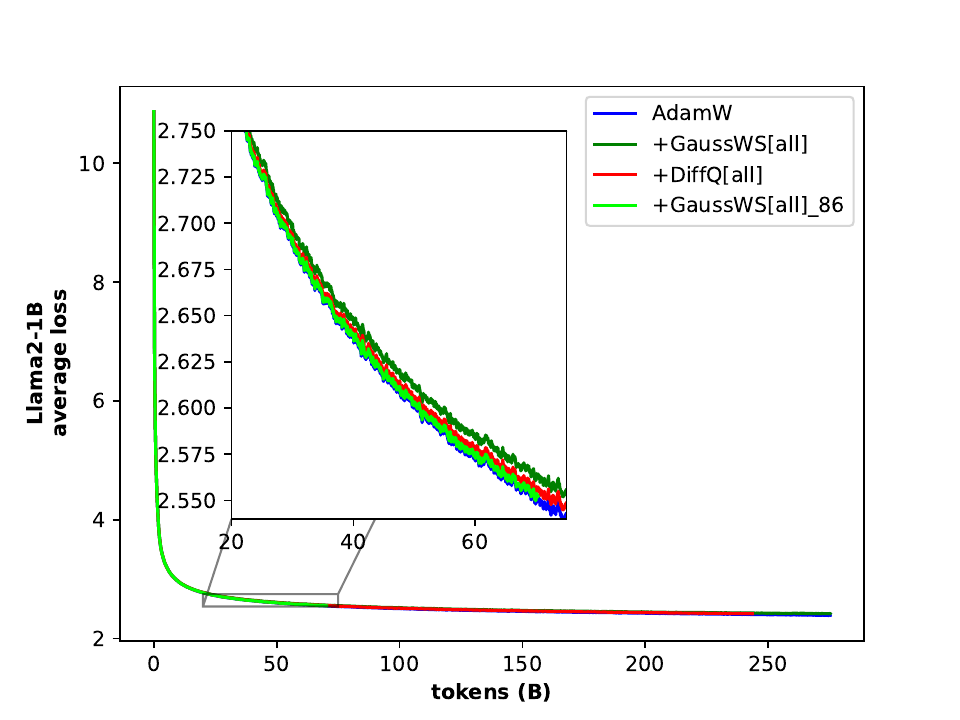}
        \label{fig:larger_avg}
    \end{subfigure}
    \begin{subfigure}{0.33\linewidth}
        \includegraphics[trim={22 8 30 20},clip,width=\linewidth]{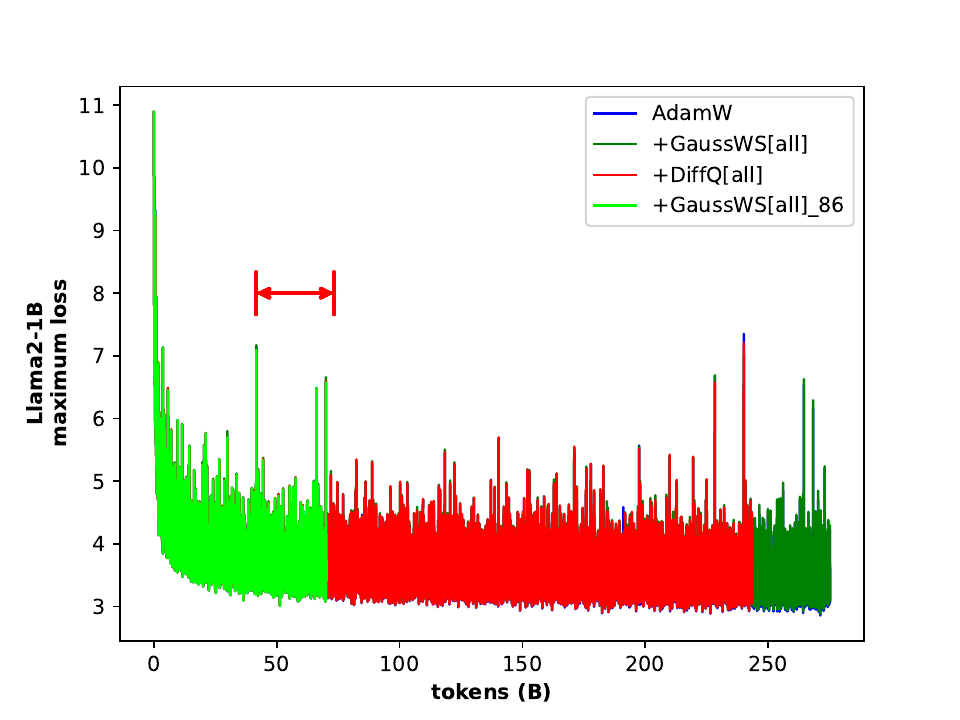}
        \label{fig:larger_max}
    \end{subfigure}
    \begin{subfigure}{0.33\linewidth}
        \includegraphics[trim={24 8 30 20},clip,width=\linewidth]{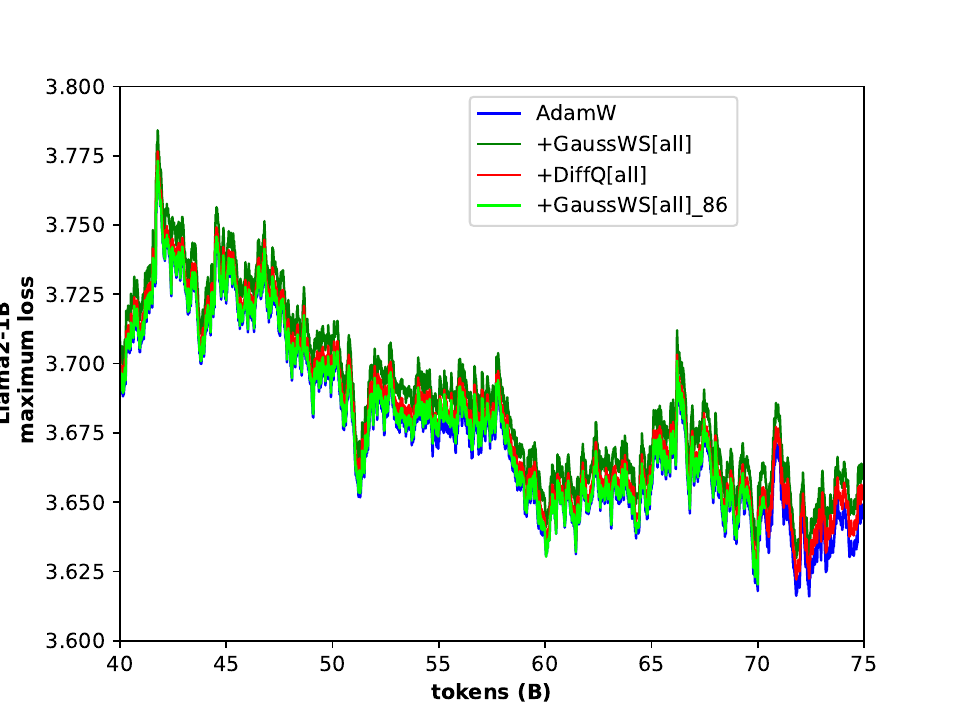}
        \label{fig:larger_max_detail}
    \end{subfigure}
    \caption{Training loss curve of the Llama2-1B model on the C4 dataset including the result with \(b_\text{init}=8\) and \(b_\text{target}=6\). First column represents average loss and the other two represent maximum loss.
    Third column corresponds to the range annotated with the orange arrow on the second column.
    For better visualization, weighted moving average is used with \(\alpha=1/16\) on left column and \(\alpha=1/128\) on right column.}
    \label{fig:larger_hyperparam}
\end{figure*}

\end{document}